\documentclass[letterpaper, 10pt, conference]{ieeeconf}
\IEEEoverridecommandlockouts
\usepackage{cite}
\usepackage{amsmath,amssymb,amsfonts}
\usepackage{graphicx}
\usepackage{booktabs}
\usepackage{multirow}
\usepackage{textcomp}
\usepackage{xcolor}

\newcommand{\method}{SCVC}
\newcommand{\YI}{Y_{\mathrm{I}}}
\newcommand{\YC}{Y_{\mathrm{C}}}

\newtheorem{lemma}{Lemma}
\newtheorem{proposition}{Proposition}
\newtheorem{assumption}{Assumption}

\begin{document}

\title{\LARGE \bf Selective Cross-View Consistency for World Action Models:\\ Held-Out Viewpoint Robustness Without Test-Time Camera Information}

\author{\authorblockN{Bingqi Huang, Bingchuan Wei, Yingkai Cai, Zhaokui Wang}
\authorblockA{Tsinghua University\\
\texttt{hbq21@mails.tsinghua.edu.cn}}}

\maketitle

\begin{abstract}
World action models (WAMs) jointly denoise future video frames and robot actions, and the video prior is expected to make their control more general. Camera viewpoint change remains one of their hardest perturbation axes. We study a question that is specific to this model class: when training with same-state cross-view image pairs, on which output coordinates should a consistency loss be imposed? The WAM denoising target mixes view-covariant coordinates, namely the predicted future scene, with view-invariant coordinates, namely the action chunk, future proprioception, and value. We show that consistency applied to the covariant block is provably harmful, shrinking legitimate view-specific content to a fraction $1/(1+4\lambda)$ of its true value, and we verify this shrinkage law in controlled experiments. Selective cross-view consistency (\method{}) therefore constrains only the invariant block, requires no camera labels, extrinsics, depth, or view synthesis at training or test time, and leaves the deployment interface unchanged. We introduce a carve-and-hold-out evaluation protocol on the LIBERO-Plus camera track that separates a distribution-matched ceiling from genuine interpolation and extrapolation to held-out viewpoints, with a matched pair-trained control isolating the effect of the consistency term from pair exposure. On held-out orbital viewpoints beyond the training envelope, \method{} improves closed-loop success over the matched control by $12.2$ points ($95\%$ CI $[7.4, 17.0]$; $+15.5$, CI $[11.7, 19.4]$, under an independent second training seed)---an effect two further camera axes replicate---while interpolation within the envelope shows no gain in either seed ($-1.2$ and $-4.3$ points) and in-distribution competence is preserved ($-0.6$, $-0.2$). We additionally report a cross-backbone audit showing that published camera-robustness numbers are strongly confounded by wrist-camera pose stability.
\end{abstract}

\section{Introduction}
\label{sec:intro}

World action models post-train large video generation backbones into robot policies that jointly predict future observations and actions~\cite{kim2026cosmos,ye2026dreamzero,wang2026wamsurvey}. The premise of the paradigm is that pretrained visual world knowledge transfers to control. On several perturbation axes this promise holds up well, yet benchmark studies consistently identify camera viewpoint as one of the hardest axes for WAMs and VLAs alike: on LIBERO-Plus, camera perturbations produce some of the largest drops across model families~\cite{fei2025liberoplus}, and a recent robustness comparison concludes that video priors offer world action models little benefit precisely when the geometric configuration of the scene is altered~\cite{zhang2026generalize}.

\begin{figure*}[t]
\centering
\includegraphics[width=0.97\textwidth]{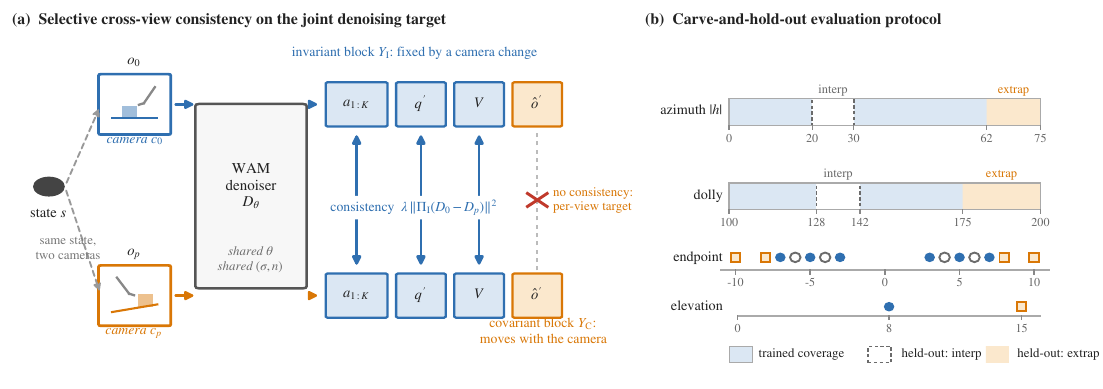}
\caption{\method{} and the carve-and-hold-out protocol. \textbf{(a)} A world action model denoises a joint target that mixes coordinates fixed by a camera change---the action chunk $a_{1:K}$, future proprioception $q'$, and value $V$ (invariant block $\YI$, blue)---with the predicted future scene $\hat o'$, which must track the camera (covariant block $\YC$, orange). \method{} renders the same state under two cameras, passes both through the shared denoiser under a shared noise draw $(\sigma, n)$, and penalizes cross-view disagreement only on $\YI$; the future scene keeps its per-view target, since forcing agreement there provably ghosts the prediction (Proposition~\ref{prop:shrink}). \textbf{(b)} Training pairs cover each benchmark camera axis except held-out regions. Dense axes are carved: an interpolation hole inside the trained range, extrapolation beyond it. Sparse axes are flipped: training uses only off-benchmark values (endpoint $\{\pm 3, \pm 5, \pm 7\}$; elevation $8$), so every official value is held out (endpoint interp $\{\pm 4, \pm 6\}$, extrap $\{\pm 8, \pm 10\}$; elevation extrap $15$). Official tasks are bucketed by their camera parameters into in-distribution, interpolation, and extrapolation buckets.}
\label{fig:concept}
\end{figure*}

Published camera-track numbers also obscure what is being measured. Most evaluated policies consume a wrist camera in addition to the third-person scene camera, and the wrist view is pose-stable by construction: it moves with the arm, so a scene-camera perturbation barely changes it. In a masking audit (Section~\ref{sec:audit}) we find that a wrist-equipped Cosmos policy still solves two-thirds of camera-perturbed tasks with its scene camera blacked out, yet collapses to near zero with its wrist blacked out; on a public $\pi_{0.5}$ policy~\cite{intelligence2025pi05} the contrast is $100\%$ against $0\%$. On the same benchmark, a wrist-equipped checkpoint loses only $19$ points under camera perturbation while its scene-only counterpart loses $67$. Published wrist-enabled camera robustness therefore conflates scene-view invariance with wrist stability. Our study consequently trains and evaluates scene-only policies, a controlled lesion that makes scene-view invariance directly measurable.

With the measurement question settled, the method question becomes visible, and it is unique to WAMs. A natural route to viewpoint robustness is consistency training on same-state image pairs rendered from different cameras: the same physical state must map to the same action regardless of the viewpoint. For a VLA, whose entire output is an action, applying such a loss is conceptually straightforward, and prior work has shown it beats training on the same pairs without the constraint~\cite{anon2026crossview}. A WAM breaks this simplicity. Its denoising target contains coordinates that must \emph{change} with the camera (the predicted future scene frame) next to coordinates that must \emph{not} (the action chunk, future proprioception, and value). Enforcing agreement across views on the wrong coordinates actively degrades the model: we show it drives the predictions toward a ghosted average of the two views' futures, shrinking the view-specific content to exactly $1/(1+4\lambda)$ of its true value at the pointwise optimum (Proposition~\ref{prop:shrink}). Selectivity is thus a structural requirement. Our method, selective cross-view consistency (\method{}), imposes the consistency penalty only on the invariant block (Fig.~\ref{fig:concept}a), keeps per-view supervision on the covariant block, and shares the diffusion noise draw across the two branches of each pair, a pairing choice we also justify analytically (Proposition~\ref{prop:noise}).

Evaluating such a method raises a second problem that deserves as much care as the method itself. Any consistency or augmentation method must train on perturbed viewpoints, so it can never claim the zero-shot setting of a nominally trained baseline. If the training camera distribution matches the benchmark distribution, high scores certify view-invariant control under a disclosed contract but say nothing about unseen viewpoints; if training viewpoints are sampled disjoint from the benchmark but concentrated near nominal, the method fails on severe views simply for lack of coverage. We encountered both failure modes in early versions of this work. Our carve-and-hold-out protocol (Section~\ref{sec:protocol}; Fig.~\ref{fig:concept}b) is the disciplined middle ground: training covers the full severity range of each camera axis except deliberately held-out bands, and the official benchmark tasks are re-partitioned by their camera parameters into in-distribution, interpolation, and extrapolation buckets. One benchmark then yields both a distribution-matched ceiling and a genuine out-of-distribution test, and a matched control trained on identical carved pairs with the consistency weight set to zero isolates what the objective adds beyond pair exposure.

This paper makes four contributions. (i) We formulate the selective consistency principle for structured denoising targets and instantiate it as \method{} on an open 2B-parameter WAM, with an analysis showing that the supervised anchor constrains only the cross-view mean while the consistency term reallocates optimization pressure onto the view-sensitive subspace (Lemma~\ref{lem:identity}), and that wrong-coordinate consistency incurs a quantified shrinkage bias that we confirm at machine precision in simulation and within $10\%$ on a controlled denoiser trained to convergence (Section~\ref{sec:verification}). (ii) We introduce the carve-and-hold-out evaluation protocol together with an explicit disclosure framework separating distribution matching from data leakage. (iii) We report controlled results on the LIBERO-Plus camera track: constraining only the invariant block yields a $+12.2$-point gain on held-out extrapolation-regime orbital viewpoints ($+15.5$ under a second training seed), replicated on two further camera axes, with no gain under within-envelope interpolation and preserved in-distribution success. (iv) We provide the cross-backbone wrist-dominance audit as a correction to current evaluation practice. Code and protocol tooling will be released.

\section{Related Work}
\label{sec:related}

\textbf{World action models and their robustness.}
Cosmos Policy~\cite{kim2026cosmos} post-trains a video diffusion backbone with actions, proprioception, and value injected as dedicated latent frames; DreamZero~\cite{ye2026dreamzero} trains a 14B autoregressive video-action model that transfers zero-shot; surveys map the growing family~\cite{wang2026wamsurvey}. Robustness for these models has so far been pursued through architecture and representation, for example object-centric decompositions~\cite{liu2026oawam}, through diagnosis of the foresight-action gap~\cite{qiu2026foresight}, and through benchmark construction~\cite{zhang2026generalize}. Closest in spirit, ReViWo~\cite{pang2025reviwo} learns a view-invariant \emph{encoder} for a latent world model contrastively while its decoder stays deliberately view-dependent; invariance there is a property of the representation, whereas we constrain the model's \emph{outputs} selectively by their transformation law. AGRA~\cite{qiu2026foresight} is the closest diagnostic neighbor: it shows that plausible predicted futures do not guarantee accurate actions and repairs the gap with representation alignment. We differ in conditioning on a controlled nuisance, separating output blocks by their transformation laws, and repairing the gap with a training objective in place of an alignment module. To our knowledge no prior work constrains how the joint video-action output transforms under camera change at the objective level.

\textbf{Viewpoint-robust visuomotor policies.}
Existing routes fall into families that are orthogonal to ours and combinable with it. Test-time view synthesis re-renders the observed view toward the training camera~\cite{heo2026anycamvla,gu2026vistabot}, changing the deployment pipeline; camera conditioning feeds explicit extrinsics to the policy~\cite{jiang2025camera}, changing the input contract. Augmentation methods synthesize novel training views~\cite{tian2024vista,chen2024roviaug,cai2026beyond}, and the finding of~\cite{cai2026beyond} that multi-view supervision helps even at a fixed test view underlines that marginal coverage and per-state constraint are different mechanisms; none of these place synthesized views inside a consistency loss. Reinforcement post-training with smoothness regularization targets noise, a different nuisance from viewpoint equivalence~\cite{zhang2025robustvla}. Equivariant policies~\cite{yang2024equibot,wang2024equivariant} hard-code analytic symmetry groups, whereas the camera nuisance here has no analytic group action on pixels and is realized only by data; moreover the selectivity problem does not arise for a single-output policy. Finally, the field disagrees about where actions should live: some systems predict actions in camera coordinates while others learn view-invariant latent actions~\cite{jeong2026vila}; the latter also evaluates on held-out interpolated and extrapolated viewpoints, on custom tasks with a representation-level objective, which corroborates the regime split we instantiate on a public benchmark's official perturbations. This disagreement is exactly why we state the transformation law of every output block explicitly and verify the base-frame action convention in code (Assumption~\ref{ass:frame}).

\textbf{Consistency regularization.}
Our objective descends from semi-supervised consistency training~\cite{laine2017temporal,tarvainen2017mean,sohn2020fixmatch} and from the theory of augmentation as orbit averaging~\cite{chen2020group}. Most relevant is the analysis of data augmentation consistency by Yang et al.~\cite{yang2023dac}, which proves that enforcing consistency on augmented pairs yields lower estimation error than empirical risk minimization on the same augmented data. Our matched control (Section~\ref{sec:protocol}) is the empirical counterpart of that comparison. The theory, however, assumes the entire label is invariant under the augmentation. The WAM target violates this assumption on the future-scene block, and Proposition~\ref{prop:shrink} quantifies the resulting bias; the noise-conditional structure of denoising adds a second matching requirement on the noise draw (Proposition~\ref{prop:noise}) that has no analogue for deterministic predictors. Cross-view action consistency for flow-matching VLAs~\cite{anon2026crossview} validated the mechanism on a fully invariant output; the present work is the extension to structured targets where selectivity becomes necessary.

\section{Problem Setup}
\label{sec:setup}

Let $s$ denote the physical state of the scene and robot, $c$ the scene-camera parameters, and $o = h(s, c)$ the rendered observation. A policy receives $(o, q, \ell)$, where $q$ is proprioception and $\ell$ a language instruction. An \emph{action-equivalent pair} is a pair of observations $(o_0, o_p) = \bigl(h(s, c_0),\, h(s, c_p)\bigr)$ rendered from the same state under a nominal and a perturbed camera; both share the demonstrated action chunk $a \in \mathbb{R}^{K \times A}$ (here $K{=}16$), the future proprioception $q'$, and the value label $V$.

The WAM denoises a composite target $Y = (\YI, \YC)$. The invariant block $\YI = (a, q', V)$ consists of coordinates fixed by a camera change; the covariant block $\YC = o' = h(s', c)$ is the future scene frame, which transforms with the camera. In the Cosmos Policy instantiation used here, the scene-only latent sequence has seven frames: a blank frame, current proprioception, the current scene image (clean conditioning), the action chunk, future proprioception, the future scene image, and value. The two branches of a pair differ only in the conditioning frame and in the content of the future-scene target; the frames of $\YI$ carry bitwise-identical targets across the branches, a property we assert at training time.

\begin{assumption}[Action coordinate convention]
\label{ass:frame}
Action chunks are delta end-effector commands expressed in the robot base frame. Under this convention $\YI$ is invariant to camera change provided the pair shares the physical state, both views are task-informative, and the perturbation does not alter task semantics. We verify the convention directly against the dataset and controller: a regression of stored actions onto base-frame end-effector displacements recovers a scaled identity map. If actions were expressed in camera coordinates, as some recent systems deliberately choose, the invariance premise would fail by construction.
\end{assumption}

View robustness requires the induced action distribution $\pi_\theta(a \mid h(s,c), q, \ell)$ to be constant in $c$ over the deployment camera range. Training on camera-diverse data enlarges the marginal support of $c$ but imposes no per-state constraint tying $\pi_\theta(a \mid h(s,c_0))$ to $\pi_\theta(a \mid h(s,c_p))$; the estimation-error separation between augmented ERM and augmentation consistency~\cite{yang2023dac} formalizes why the constraint should add value on identical data, and held-out viewpoints are where the difference should be largest, since a per-state equivalence has more reason to generalize beyond the sampled cameras than marginal coverage does.

\section{Selective Cross-View Consistency}
\label{sec:method}

\subsection{Objective}
\label{sec:objective}

The backbone is an EDM-style denoiser~\cite{karras2022edm} $D_\theta(x_\sigma; \sigma, \mathrm{cond}(o_v), \ell)$ producing $x_0$ predictions over the latent frame sequence, trained with per-frame weight $w(\sigma)$. For every pair we draw one noise level and one noise realization $(\sigma, n)$ and apply them to both branches, which then differ only in the scene conditioning frame and the covariant target content. Writing $D_v$ for the prediction of branch $v \in \{0, p\}$, $x^v$ for its target, and $\Pi_{\mathrm{I}}$ for the selector of invariant frames, the loss is
\begin{equation}
\label{eq:scvc}
\mathcal{L} = \tfrac{1}{2}\!\!\sum_{v \in \{0,p\}}\!\! w(\sigma)\,\bigl\lVert D_v - x^v \bigr\rVert^2
\;+\; \lambda(t)\, w(\sigma)\,\bigl\lVert \Pi_{\mathrm{I}}\,(D_0 - D_p) \bigr\rVert^2,
\end{equation}
with $\lambda(t)$ ramped linearly to $\lambda_{\mathrm{CV}} = 2.0$ over the first $2\%$ of training. The covariant frame receives only the per-view supervised term, so each branch remains free to predict its own camera's future. Both branches contribute the supervised term regardless of $\lambda$, so the $\lambda{=}0$ configuration is a camera-augmented control; the evaluation design is built on this comparison. The consistency term applies to demonstration samples, for which same-state pairs can be rendered from recorded simulator states; rollout-mode samples pass through single-view. No camera parameters, extrinsics, depth, or synthesized views enter the model at training or test time; the pair structure exists only in the loss.

\subsection{Why selectivity is necessary}
\label{sec:theory}

Three short results justify the design; each proof is elementary, so we give it inline.

\begin{lemma}[Mean-residual decomposition]
\label{lem:identity}
For predictions $D_0, D_p$ with a common target $u$, let $\bar{D} = \tfrac{1}{2}(D_0 + D_p)$ and $\delta = \tfrac{1}{2}(D_0 - D_p)$. Then
\begin{align*}
\tfrac{1}{2}\bigl(\lVert D_0 - u\rVert^2 + \lVert D_p - u\rVert^2\bigr) &+ \lambda \lVert D_0 - D_p\rVert^2 \\
&= \lVert \bar{D} - u\rVert^2 + (1 + 4\lambda)\lVert \delta \rVert^2 .
\end{align*}
\end{lemma}

\emph{Proof.} Writing $D_0 = \bar{D} + \delta$ and $D_p = \bar{D} - \delta$, the two supervised terms sum to $2\lVert\bar{D}-u\rVert^2 + 2\lVert\delta\rVert^2$ because the cross terms $\pm 2\langle\bar{D}-u,\delta\rangle$ cancel; halving gives $\lVert\bar{D}-u\rVert^2 + \lVert\delta\rVert^2$, and the consistency term is $\lambda\lVert 2\delta\rVert^2 = 4\lambda\lVert\delta\rVert^2$. $\square$

The supervised anchor constrains only the cross-view mean, and the consistency weight scales only the view-residual penalty. On invariant coordinates, where both branches share the target $u$, the term is therefore unbiased pressure toward view agreement.

\begin{proposition}[Wrong-coordinate shrinkage]
\label{prop:shrink}
Apply the objective of Lemma~\ref{lem:identity} to a block whose targets differ across views, $u_0 \neq u_p$. The pointwise minimizer satisfies $\bar{D}^\ast = \bar{u}$ and
\begin{equation*}
(D_0 - D_p)^\ast = \frac{u_0 - u_p}{1 + 4\lambda},
\qquad
D_v^\ast - u_v = \mp\,\frac{2\lambda}{1+4\lambda}\,(u_0 - u_p).
\end{equation*}
\end{proposition}

\emph{Proof.} The objective $\tfrac{1}{2}\lVert D_0-u_0\rVert^2 + \tfrac{1}{2}\lVert D_p-u_p\rVert^2 + \lambda\lVert D_0-D_p\rVert^2$ is convex; its first-order conditions $(D_0-u_0)+2\lambda(D_0-D_p)=0$ and $(D_p-u_p)-2\lambda(D_0-D_p)=0$ sum to $\bar{D}^\ast=\bar{u}$ and subtract to $(D_0-D_p)^\ast=(u_0-u_p)/(1+4\lambda)$; substituting into $D_0^\ast=\bar{u}+\tfrac{1}{2}(D_0-D_p)^\ast$ gives the per-branch bias. $\square$

Consistency on the covariant block is thus adversarial to the supervised objective: the predicted futures of the two views are pulled toward a ghosted average, with a per-view bias approaching half the true view difference as $\lambda \to \infty$. This is the structural reason a WAM cannot reuse the VLA recipe unchanged, and it yields a falsifiable quantitative prediction that we test below.

\begin{proposition}[Noise pairing]
\label{prop:noise}
Fix $\sigma$ and let $D_v(n)$ denote the branch prediction under noise realization $n$. With independent draws $n, n'$ across branches, the consistency term satisfies
$\mathbb{E}_{n,n'} \lVert D_0(n) - D_p(n')\rVert^2 = \mathbb{E}_{n} \lVert D_0(n) - D_p(n)\rVert^2 + 2\,\mathrm{tr}\,\mathrm{Cov}_n\bigl(D_0(n), D_p(n)\bigr)$.
At exact view invariance $D_0 \equiv D_p \equiv D$ the added term equals $2\,\mathrm{tr}\,\mathrm{Var}_n(D) > 0$: the independent-noise objective cannot reach zero even for a perfectly invariant denoiser, and its residual gradient penalizes the legitimate dependence of a correct denoiser on its noisy input.
\end{proposition}

\emph{Proof.} Expanding both squared norms, the $\mathbb{E}\lVert D_v\rVert^2$ terms coincide across the shared and the independent draws because $n,n'$ share a distribution, leaving $2\bigl(\mathbb{E}\langle D_0(n),D_p(n)\rangle - \langle\mathbb{E}D_0,\mathbb{E}D_p\rangle\bigr) = 2\,\mathrm{tr}\,\mathrm{Cov}_n(D_0,D_p)$; at $D_0\equiv D_p\equiv D$ this equals $2\,\mathrm{tr}\,\mathrm{Var}_n(D)\ge 0$, strictly positive whenever the denoiser depends on its input. The covariance term is not globally signed; we use the positive-sign conclusion only in the near-invariant regime. $\square$

Sharing $(\sigma, n)$ across branches is therefore part of the augmentation matching itself. Collectively, the comparator is sound exactly when its two arms match in state, coordinates, and noise, and differ only in the camera.

\subsection{Controlled verification of the shrinkage law}
\label{sec:verification}

Proposition~\ref{prop:shrink} describes pointwise minimizers; whether stochastic training of a shared network reaches them is an empirical question, and it is the question that matters, since the law is our evidence that selectivity is necessary. We test it at two scales. An exact simulation with independent per-pair minimizers reproduces the predicted ratio to machine precision, validating the derivation and the measurement code. A controlled experiment then trains a shared-trunk denoiser to convergence on synthetic two-view targets with an exactly invariant and an exactly covariant block, applying the consistency term to the covariant block, and measures the normalized view-difference ratio $R(\lambda)/R(0)$ on held-out pairs. The measured ratios are $0.715$, $0.337$, and $0.116$ at $\lambda \in \{0.1, 0.5, 2.0\}$, against predicted values $0.714$, $0.333$, and $0.111$: within ten percent everywhere and monotone in $\lambda$. Shared-network training reaches the per-pair optimum here, so wrong-coordinate consistency collapses covariant content exactly as the algebra says. A full quantitative measurement of the law on the 2B model is out of scope; as a qualitative on-model check we trained a short ($2$k-step) wrong-coordinate arm on the WAM itself, identical to the selective run except that the consistency term is \emph{additionally} imposed on the covariant future-scene block. Its training covariant guard---the held-out cross-view future-scene ratio---collapses to $0.18$ at $2$k steps, descending toward the Proposition~\ref{prop:shrink} asymptote, while the selective run and the $\lambda_{\mathrm{CV}}{=}0$ control both hold near $0.98$. Decoding the predicted future scene from held-out pairs makes the collapse visible (Fig.~\ref{fig:ewc}): on off-nominal viewpoints the wrong-coordinate model blurs its prediction toward a view-averaged image, ghosting the scene, while the selective run keeps distinct, view-specific futures.

\begin{figure}[t]
\centering
\includegraphics[width=\linewidth]{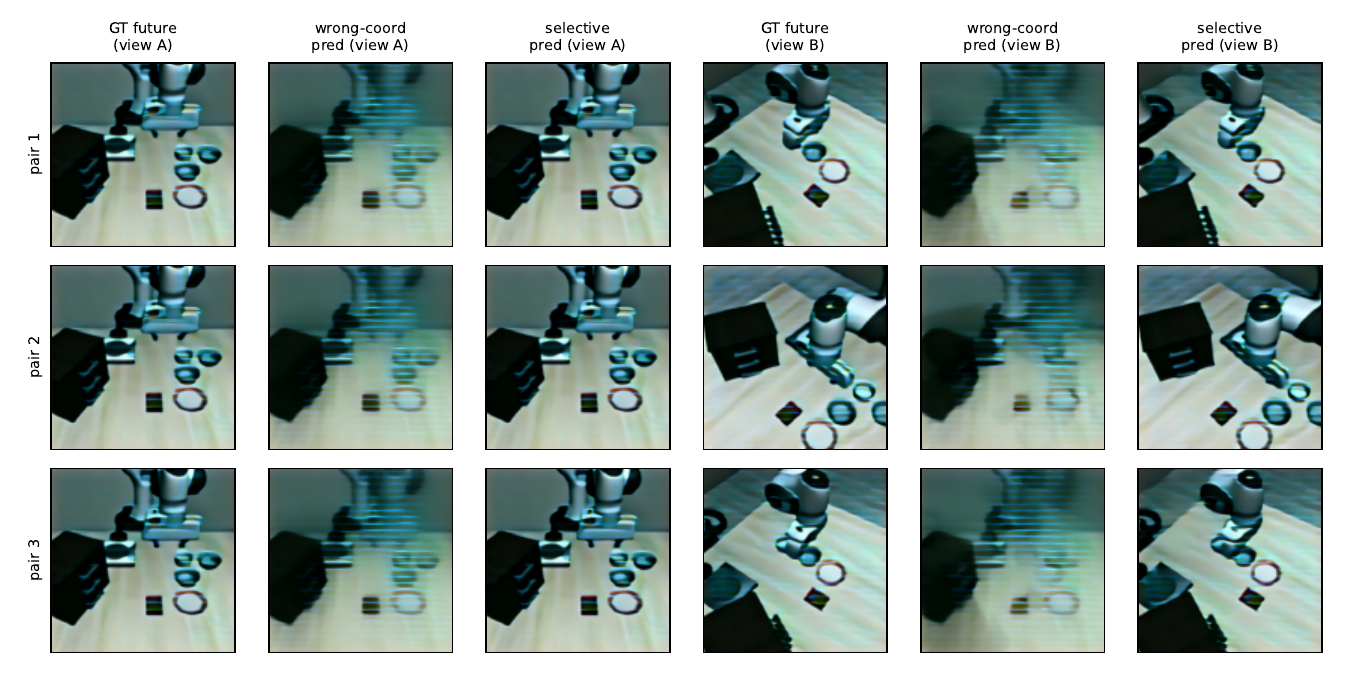}
\caption{Wrong-coordinate consistency ghosts the predicted future scene (2B model, both arms at $2$k steps, shared noise draw). For three held-out pairs we decode the future-scene prediction under two cameras. On the near-nominal view A both arms track the ground-truth future; on the off-nominal view B the wrong-coordinate arm collapses its prediction toward a view-averaged, ghosted image, whereas the selective run preserves the view-specific future.}
\label{fig:ewc}
\end{figure}

\section{A Carve-and-Hold-Out Evaluation Protocol}
\label{sec:protocol}

\subsection{Why the protocol is needed}

A consistency method must train on perturbed viewpoints, which rules out the zero-shot evaluation available to a nominally trained policy. The two obvious alternatives both mislead. Sampling training cameras disjoint from the benchmark but concentrated near nominal produces models that fail on severe viewpoints for lack of severity coverage; an early version of our training set had no pairs beyond $31^\circ$ of azimuth and every method variant floored on severe cells for that reason alone. Matching the training camera distribution to the benchmark, the disclosed contract used in prior cross-view work~\cite{anon2026crossview}, produces strong scores that certify view-invariant control in distribution but leave no room for an out-of-distribution claim, because every evaluated pose was trained on. The protocol below keeps the disclosed distribution-matched contract where it is honest and adds held-out structure where an OOD claim needs it. Throughout, the leakage question is kept separate from the pose question: training images are re-rendered from recorded demonstration states while evaluation runs fresh rollouts from benchmark initial states, so states, rollouts, and labels never cross between training and evaluation regardless of camera overlap.

\subsection{Carving and the sparse-axis flip}

LIBERO-Plus enumerates $1{,}599$ camera-perturbed task instances whose parameters span four axes: orbital azimuth (horizon, $\pm 75^\circ$), elevation (vertical, $\{0, 15\}$), dolly distance (scale, $100$ to $200$), and in-place reorientation of the camera endpoint ($\pm 10$)~\cite{fei2025liberoplus}. Our training pairs sample the same axes with density matched to the benchmark marginal, after which we remove held-out regions. On the dense axes we carve bands: azimuth training keeps $|h| \in [0, 62]$ minus an interpolation hole at $(20, 30)$, leaving $(62, 75]$ as extrapolation; scale keeps $[100, 175]$ minus a hole at $(128, 142)$, leaving $(175, 200]$. On the sparse axes carving is impossible, since the benchmark uses a handful of discrete values, so we flip: all official endpoint values are removed from training and replaced by self-rendered pairs at the non-official values $\{\pm 3, \pm 5, \pm 7\}$, making official $\{\pm 4, \pm 6\}$ interpolation and $\{\pm 8, \pm 10\}$ extrapolation; elevation training uses the non-official value $8$, making the official $15$ a single extrapolation point. Evaluation thus consists entirely of official benchmark tasks on every axis.

The final integrated training manifest contains $583{,}648$ same-state pairs. An automated guard asserts that no held-out value appears in any training row, and the severe-azimuth range $[31, 62]$ retains $66{,}437$ pairs, closing the coverage failure mode described above. At the fixed training budget the manifest is traversed about $6.2$ times, below our pre-set overfitting threshold of ten effective epochs.

\subsection{Buckets, controls, and pre-registered analysis}

Parsing the camera parameters of every official task assigns it to one of twelve axis-level buckets. Buckets inside trained bands form the distribution-matched ceiling; buckets in holes are interpolation and buckets beyond the trained range are extrapolation, both genuinely unseen. The nine single-axis buckets appear with their sizes in Table~\ref{tab:main}; two small compound buckets and the $\pm 2$ endpoint bucket are reported in the text only, for completeness.

Two models are trained on the identical carved manifest from the same scene-only base checkpoint: a control with $\lambda_{\mathrm{CV}} = 0$ and the method (\method{}) with $\lambda_{\mathrm{CV}} = 2.0$. Because the control receives the supervised loss on both branches, it is a camera-augmented baseline, and the method-minus-control difference isolates the consistency term under identical data, initialization, and budget. Our analysis plan was fixed before evaluation: the primary endpoints are the azimuth interpolation and extrapolation buckets, compared by paired task bootstrap with $95\%$ confidence intervals; the remaining single-axis buckets are secondary; in-distribution buckets carry no OOD language. We report the plan's outcome as measured, including null results.

\section{Experiments}
\label{sec:experiments}

\subsection{Setup}
\label{sec:exp-setup}

All models continue post-training from the released Cosmos Policy LIBERO checkpoint~\cite{kim2026cosmos} under a scene-only configuration. A stable scene-only reference is first trained on nominal LIBERO data~\cite{liu2023libero} and frozen; both carved-manifest models then initialize from it. Every model trains for the same budget of $7.2$ million sample presentations ($10{,}000$ steps at effective batch $720$, learning rate $5{\times}10^{-5}$, EMA), and each reported policy is the single final checkpoint of its run. Evaluation uses the deterministic protocol of the released codebase with ten trials per task for bucketed comparisons. In-distribution skill is measured on the four LIBERO suites with $500$ episodes per suite.

\subsection{What camera-track numbers measure: the wrist audit}
\label{sec:audit}

Figure~\ref{fig:audit} presents a masking audit of the released wrist-enabled Cosmos policy run directly on the camera-perturbed tasks ($1{,}200$ rollouts per condition). Blacking out the scene camera, the very sensor the camera benchmark perturbs, lowers success only from $80.5\%$ (unmasked, Wilson $95\%$ CI $[78.2, 82.6]$) to $66.6\%$ ($[63.9, 69.2]$), with a scene-noise fill at $69.7\%$ ($[67.0, 72.2]$); blacking out the wrist collapses it to $0.2\%$ ($[0.1, 0.7]$). The policy thus still solves two-thirds of camera-perturbed tasks with no scene view at all, while it is helpless without the wrist. We rest the argument on the scene-mask direction: succeeding without an input proves the input is unnecessary, whereas failure without an input can also reflect sensitivity to the corruption itself. A smaller pilot on nominal tasks shows the same asymmetry and is stable across fill types ($140/180$ scene-masked versus $37/180$ wrist-masked, with the scene result varying by at most $5$ points across black, gray, and noise fills), and a public $\pi_{0.5}$ LIBERO policy is likewise unaffected by scene masking and completely broken by wrist masking~\cite{intelligence2025pi05,anon2026crossview}.

The consequence on the benchmark is a systematic confound (Table~\ref{tab:audit}). Between the wrist-enabled released checkpoint and our matched scene-only reference, in-distribution success differs by about five points while camera-track success differs by more than fifty: the wrist-equipped policy loses $19.0$ points under camera perturbation where the scene-only policy loses $67.2$. The official leaderboard corroborates this with an independent model family, where removing wrist input from OpenVLA-OFT drops its camera score from $56.4$ to $10.4$~\cite{fei2025liberoplus}. The scene view still informs wrist-equipped policies---masking it costs $14$ points aggregate, ranging from a gain on \texttt{spatial} to a large drop on \texttt{goal}---and the audit shows the wrist view is a confound for measuring scene-view invariance, the quantity a camera benchmark intends to measure. All mechanism experiments in this paper are therefore scene-only.

\begin{figure}[t]
\centering
\includegraphics[width=0.92\linewidth]{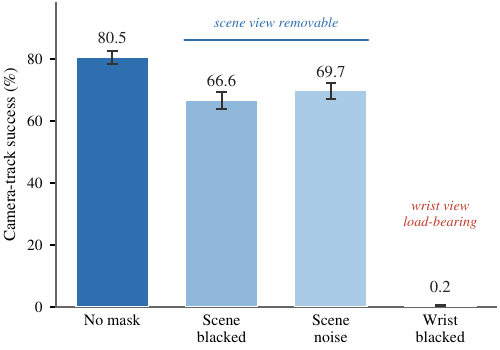}
\caption{Masking audit of the released wrist-enabled Cosmos policy on the camera-perturbed track ($1{,}200$ rollouts per condition; error bars are Wilson $95\%$ CIs). Removing the scene camera---the sensor the benchmark perturbs---leaves two-thirds of tasks solvable; removing the wrist collapses success to near zero. The camera score is therefore carried by the pose-stable wrist view.}
\label{fig:audit}
\end{figure}

\begin{table}[t]
\caption{Consequence of the wrist confound on LIBERO-Plus. ID and Camera are success rates in \%; Drop is their difference. Between the wrist-enabled release and our matched scene-only reference, ID differs by five points but the camera drop differs by nearly fifty. OFT numbers (wrist\,/\,no-wrist) are from the official leaderboard~\cite{fei2025liberoplus}.}
\label{tab:audit}
\centering
\small
\setlength{\tabcolsep}{5pt}
\begin{tabular}{lccc}
\toprule
Policy & ID & Camera & Drop \\
\midrule
Released ckpt (wrist) & $98.5$ & $79.5$ & $19.0$ \\
Scene-only reference & $93.2$ & $26.0$ & $67.2$ \\
OFT (wrist\,/\,none) & -- & $56.4$\,/\,$10.4$ & -- \\
\bottomrule
\end{tabular}
\end{table}

\subsection{Motivating observation: the dream survives where the action fails}
\label{sec:dissociation}

The scene-only reference reaches $93.2\%$ in distribution yet $26.0\%$ on the full camera track, and the drop is uneven across axes: $7.5\%$ under dolly, $21.4\%$ under orbital viewpoint change, $61.3\%$ under in-place reorientation. This is the zero-shot gap the method addresses, and the orbital axis, which contributes the majority of the benchmark's camera instances, is where most of it lives. Before asking whether consistency helps, it is worth locating the failure. On the frozen scene-only reference we measured both sides of the model across the full camera track: closed-loop action success, and the fidelity of the predicted future scene via a camera-conditioned excess-FID that subtracts an oracle floor computed from ground-truth replays under the same camera. On the orbital axis, action success collapses from $53\%$ to $3\%$ across severity levels while relative excess-FID stays flat (between $-0.07$ and $+0.16$ across all levels). On the dolly axis, in contrast, prediction fidelity co-degrades with control. The video prior thus survives precisely on the axis where control fails hardest, echoing foresight-action misalignment observed elsewhere~\cite{qiu2026foresight}, and marking the orbital axis as the natural primary target for an action-side repair. We use this as motivation; the excess-FID certifies distributional fidelity of the predicted frames, a strictly weaker property than semantic plan correctness.

\begin{figure}[t]
\centering
\includegraphics[width=0.94\linewidth]{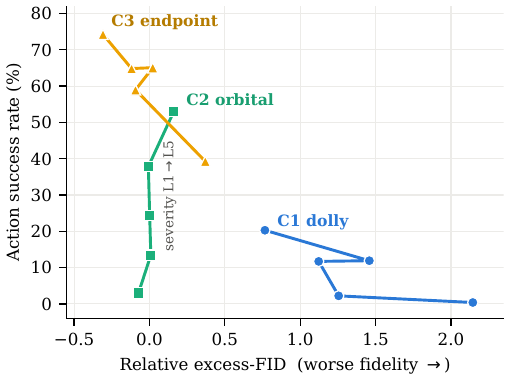}
\caption{Fidelity-control dissociation on the scene-only reference. On the orbital axis (C2) the predicted future scene retains its fidelity while closed-loop success collapses; on the dolly axis (C1) both degrade together; in-place reorientation (C3) stays near baseline fidelity. Each marker is one severity level (L1--L5); the dashed line is zero relative excess-FID (fidelity matched to the oracle floor).}
\label{fig:e1}
\end{figure}

\subsection{Main result: held-out viewpoint buckets}
\label{sec:main-results}

Table~\ref{tab:main} reports the pre-registered comparison.

The primary comparison is mixed, and the split falls exactly on the interpolation--extrapolation boundary. On the two pre-registered orbital (C2 azimuth) endpoints, \method{} leaves interpolation unchanged---$85.1$ vs.\ $86.3$, a paired difference of $-1.22$ points whose $95\%$ CI $[-4.29, +1.84]$ straddles zero---while improving extrapolation to viewpoints beyond the training envelope by $+12.20$ points ($[+7.40, +17.00]$, excluding zero). We scope the claim to this extrapolation regime and report the interpolation result as a null. Two further camera axes, pre-registered as secondary endpoints, replicate the pattern independently: dolly-scale extrapolation improves by $+4.24$ ($[+0.71, +7.88]$) and elevation extrapolation by $+8.76$ ($[+6.40, +11.11]$), while both interpolation counterparts are null; the C3 in-place-reorientation axis is null on both sides. Three independent perturbation axes thus share one signature: no effect where held-out views interpolate within the training envelope, and a positive, CI-separated effect where they extrapolate beyond it. The gains do not trade against in-distribution competence: full four-suite ID differs by only $-0.60$ points, within our pre-registered preservation tolerance, and the C1 in-distribution ceiling is flat ($-0.18$); the distribution-matched C2 azimuth region shows a descriptive $+4.52$-point gain, which we classify as distribution-matched improvement since those views lie inside the training envelope. The three buckets omitted from Table~\ref{tab:main} are the two compound azimuth-and-elevation buckets, null here ($+0.94$, CI $[-2.34, +4.38]$; $+4.85$, CI $[-2.63, +11.82]$), and the six-task $\pm 2$ endpoint bucket, too small to read ($-5.00$, CI $[-13.33, +3.33]$). We interpret the signature, and the places where it disappears---including the suite-level regression hiding inside the compound null---in Section~\ref{sec:discussion}.

A second seed, both arms retrained under the identical protocol, reproduces the hierarchy: extrapolation $+15.50$ (CI $[+11.70, +19.40]$), dolly $+4.44$ ($[+0.61, +8.49]$), elevation $+4.84$ ($[+3.11, +6.58]$), the distribution-matched azimuth region $+3.40$ ($[+1.77, +4.93]$), ID preservation $-0.2$, and both compound buckets now positive ($+5.16$, $[+0.78, +9.69]$; $+16.36$, $[+12.73, +20.20]$). Its interpolation cell reads $-4.29$ ($[-7.56, -1.22]$). Across seeds the method's interpolation successes move by $3$ trials in $490$ ($85.1\%$ to $85.7\%$) and the control's by $18$ ($86.3\%$ to $90.0\%$); $14$ of those $18$ lie in the five interpolation views of one long-horizon task whose success varies strongly between checkpoints, and the \texttt{goal} suite's control solves $100\%$ of interpolation trials in both seeds (the method: $95.6\%$, $90.0\%$).

\begin{table}[t]
\caption{Held-out viewpoint evaluation. Both models are trained on the identical carved pair manifest from the same base checkpoint; the only difference is the consistency weight. Success rates in \%, ten trials per task; $\Delta$ is the paired task-bootstrap difference with $95\%$ CI. Primary endpoints in bold. In-distribution buckets are a distribution-matched ceiling, not an OOD claim. Numbers are the first training seed; the second seed's twelve-bucket table is reported in the text and released with the code.}
\label{tab:main}
\centering
\setlength{\tabcolsep}{4pt}
\resizebox{\columnwidth}{!}{%
\begin{tabular}{llrccc}
\toprule
Axis & Bucket & $N$ & Control & \method{} & $\Delta$ (95\% CI) \\
\midrule
\multirow{3}{*}{C2 azimuth} & in-dist & 294 & $84.0$ & $88.5$ & $+4.52$~($+2.65,+6.39$) \\
 & \textbf{interp} & 49 & $86.3$ & $85.1$ & $\mathbf{-1.22}$~($-4.29,+1.84$) \\
 & \textbf{extrap} & 100 & $63.8$ & $76.0$ & $\mathbf{+12.20}$~($+7.40,+17.00$) \\
\midrule
\multirow{3}{*}{C1 dolly} & in-dist & 168 & $91.5$ & $91.3$ & $-0.18$~($-1.90,+1.67$) \\
 & interp & 46 & $88.7$ & $89.6$ & $+0.87$~($-1.96,+3.91$) \\
 & extrap & 99 & $67.0$ & $71.2$ & $+4.24$~($+0.71,+7.88$) \\
\midrule
C2 elevation & extrap & 386 & $54.9$ & $63.7$ & $+8.76$~($+6.40,+11.11$) \\
\midrule
\multirow{2}{*}{C3 endpoint} & interp & 62 & $92.6$ & $91.6$ & $-0.97$~($-3.71,+1.61$) \\
 & extrap & 226 & $88.6$ & $88.0$ & $-0.58$~($-2.12,+1.02$) \\
\midrule
\multicolumn{2}{l}{ID (4 suites, 500 eps/suite)} & & $92.3$ & $91.7$ & $-0.60$ \\
\bottomrule
\end{tabular}}
\end{table}

\section{Discussion}
\label{sec:discussion}

\textbf{What the evidence attributes to the objective.} The comparison isolates the consistency term from pair exposure: the control and \method{} share the base checkpoint, the carved manifest, the initialization, the budget, and the noise schedule, and differ only in $\lambda_{\mathrm{CV}}$. The control saw exactly the same perturbed cameras, so the extrapolation gain is attributable to the consistency term alone. The three-axis signature---null within the training envelope, positive and CI-separated beyond it---is the pattern the mechanism predicts: a per-state cross-view equivalence has more reason to transport to unsampled cameras than marginal coverage does (Section~\ref{sec:setup}), so its benefit should surface precisely where coverage runs out. The interpolation cells fit this reading: inside the narrow holes both models sit at $85$--$93\%$ and leave little to recover, whereas the extrapolation buckets begin at $55$--$67\%$; what varies there across seeds is the control ($86.3$ to $90.0$) while the method holds ($85.1$, $85.7$). The same logic covers the one in-distribution effect: the distribution-matched azimuth region is the only trained region below ceiling ($84\%$, because it includes severe trained views), and there the constraint also helps ($+4.52$)---the per-state equivalence pays wherever severity leaves headroom, and the held-out buckets isolate the part of that benefit which transports beyond coverage.

Two further checks constrain the interpretation from opposite sides. Destroying the state matching with a shuffled-pair control, in the flow-VLA instantiation of the same objective, removes the benefit~\cite{anon2026crossview}, ruling out generic regularization; and the controlled shrinkage-law verification (Section~\ref{sec:verification}) shows the wrong-coordinate variant failing exactly as the algebra predicts, so selectivity is doing identifiable work.

\textbf{Where the effect stops, and why.} The gain is real but bounded and uneven, and the boundaries are informative. The near-null on the C3 endpoint axis is expected: its held-out perturbation spans only $\pm 10^\circ$, small enough that the control already generalizes, leaving little invariance for the objective to add. The heterogeneity inside the primary extrapolation bucket is sharper---\texttt{spatial} improves by $40.8$ points while \texttt{object} regresses by $11.6$---and under a compound extreme azimuth-and-elevation shift outside the pre-registered endpoints the \texttt{object} suite regresses by $40.4$ points, a failure no single-axis bucket approaches. The pattern admits a consistent reading: cross-view agreement on the invariant block constrains where the end-effector should go, while the fine grasp geometry that the object-centric suites stress lies outside its reach, so on tasks dominated by grasp-precision failures the objective has little to repair---and the compound-shift regression shows the trade can go negative. On this reading, output-level selectivity addresses the viewpoint-disorientation component of the zero-shot gap, which dominates the orbital axis identified in Section~\ref{sec:dissociation}, and leaves the manipulation-precision component to complementary methods. The exchange the objective offers is asymmetric: inside the envelope it buys nothing and can cost a few points against a near-ceiling control; beyond the envelope, where the control falls to $55$--$67\%$, it returns $+4$ to $+16$ points across three axes and two seeds.

\section{Limitations}
\label{sec:limitations}

The method requires same-state cross-view pairs, which are cheap in simulation through state-reset rendering but demand synchronized multi-camera rigs or paired datasets in the real world; extending the analysis to approximately matched pairs is future work, and the covariance decomposition behind Proposition~\ref{prop:noise} suggests a route. All closed-loop evidence in this paper is likewise simulated, on a recognized public benchmark under a disclosed protocol; the flow-VLA instantiation of the same consistency mechanism has transferred to a real robot~\cite{anon2026crossview}, which speaks to the mechanism's transferability; real-robot evidence for the WAM instantiation remains open. Consistency is supervised only on demonstration states, since rollout data lacks recoverable simulator states, so we make no claim about view robustness on failure or recovery states. The elevation axis contributes a single extrapolation point by construction of the benchmark. Both arms were retrained under an independent second seed; the extrapolation gains, the ceiling effect, and ID preservation reproduce, and the interpolation cell reads $-1.22$ (null) in one seed and $-4.29$ (CI $[-7.56, -1.22]$) in the other. Interpolation buckets are small (46--49 tasks) and sit at $85$--$93\%$ success, where checkpoint-to-checkpoint variation exceeds the paired sampling interval. The scene-only protocol is a controlled lesion for isolating scene-view invariance, and wrist-camera pose perturbation is outside our scope. Finally, the shrinkage law is verified quantitatively at controlled scale and demonstrated qualitatively on the 2B model itself (Fig.~\ref{fig:ewc}); its quantitative measurement at scale, and the equally interesting question of how much view sensitivity survives inside the network when only outputs are constrained, are deferred.

\section{Conclusion}
\label{sec:conclusion}

We asked where cross-view consistency belongs in a world action model and answered with a selectivity principle: constrain the coordinates the camera cannot move, supervise the rest per view, and match everything but the nuisance, including the noise draw. The wrong-coordinate alternative actively harms the model, with a quantitative failure law confirmed under controlled training. We paired the method with an evaluation protocol that separates distribution-matched competence from held-out generalization on a public benchmark, with a matched control that prices pair exposure separately from the objective. This selectivity buys measurable robustness to held-out viewpoints in the extrapolation regime---beyond the envelope the training pairs cover---replicated across three camera axes, while leaving interpolation and in-distribution competence intact. The protocol, the audit, and the selectivity analysis apply to any WAM whose denoising target mixes invariant and covariant blocks, beyond the backbone studied here.

\bibliographystyle{IEEEtran}
\bibliography{refs}

\begin{thebibliography}{10}
\providecommand{\url}[1]{#1}
\csname url@samestyle\endcsname
\providecommand{\newblock}{\relax}
\providecommand{\bibinfo}[2]{#2}
\providecommand{\BIBentrySTDinterwordspacing}{\spaceskip=0pt\relax}
\providecommand{\BIBentryALTinterwordstretchfactor}{4}
\providecommand{\BIBentryALTinterwordspacing}{\spaceskip=\fontdimen2\font plus
\BIBentryALTinterwordstretchfactor\fontdimen3\font minus
  \fontdimen4\font\relax}
\providecommand{\BIBforeignlanguage}[2]{{%
\expandafter\ifx\csname l@#1\endcsname\relax
\typeout{** WARNING: IEEEtran.bst: No hyphenation pattern has been}%
\typeout{** loaded for the language `#1'. Using the pattern for}%
\typeout{** the default language instead.}%
\else
\language=\csname l@#1\endcsname
\fi
#2}}
\providecommand{\BIBdecl}{\relax}
\BIBdecl

\bibitem{kim2026cosmos}
M.~J. Kim, Y.~Gao, T.-Y. Lin, Y.-C. Lin, Y.~Ge, G.~Lam, P.~Liang, S.~Song,
  M.-Y. Liu, C.~Finn, and J.~Gu, ``Cosmos policy: Fine-tuning video models for
  visuomotor control and planning,'' \emph{arXiv preprint arXiv:2601.16163},
  2026.

\bibitem{ye2026dreamzero}
S.~Ye, Y.~Ge, K.~Zheng, S.~Gao, S.~Yu, G.~Kurian, S.~Indupuru, Y.~L. Tan,
  C.~Zhu, J.~Xiang, A.~Malik, K.~Lee \emph{et~al.}, ``World action models are
  zero-shot policies,'' \emph{arXiv preprint arXiv:2602.15922}, 2026.

\bibitem{wang2026wamsurvey}
S.~Wang, J.~Shi, Z.~Fu, X.~He, F.~Liu, C.~Yang, Y.~Zhou, Z.~Fei, J.~Gong,
  J.~Fu, M.~Z. Shou, X.~Huang, X.~Qiu, and Y.-G. Jiang, ``World action models:
  The next frontier in embodied {AI},'' \emph{arXiv preprint arXiv:2605.12090},
  2026.

\bibitem{fei2025liberoplus}
S.~Fei, S.~Wang, J.~Shi, Z.~Dai, J.~Cai, P.~Qian, L.~Ji, X.~He, S.~Zhang,
  Z.~Fei, J.~Fu, J.~Gong, and X.~Qiu, ``{LIBERO}-plus: A progressive robustness
  benchmark for visual-language-action models,'' in \emph{IEEE/CVF Conference
  on Computer Vision and Pattern Recognition (CVPR)}, 2026.

\bibitem{zhang2026generalize}
Z.~Zhang, Z.~Li, B.~Rahmati, R.~H. Yang, Y.~Ma, A.~Rasouli, S.~Pakdamansavoji,
  Y.~Wu, L.~Zhang, T.~Cao, F.~Wen, X.~Wang, X.~Quan, and Y.~Zhang, ``Do world
  action models generalize better than {VLAs}? {A} robustness study,''
  \emph{arXiv preprint arXiv:2603.22078}, 2026.

\bibitem{intelligence2025pi05}
{Physical Intelligence}, K.~Black, N.~Brown, J.~Darpinian, K.~Dhabalia,
  D.~Driess \emph{et~al.}, ``$\pi_{0.5}$: a vision-language-action model with
  open-world generalization,'' \emph{arXiv preprint arXiv:2504.16054}, 2025.

\bibitem{anon2026crossview}
B.~Huang, B.~Wei, X.~Wang, Y.~Cai, and Z.~Wang, ``Cross-view action
  consistency for camera-robust vision-language-action policies,''
  \emph{arXiv preprint arXiv:2608.XXXXX}, 2026.

\bibitem{liu2026oawam}
Y.~Liu, P.~Sun, S.~Li, Y.~Xie, L.~Zhang, X.~Chao, S.~Dong, F.~Chen, X.-P.
  Zhang, and W.~Ding, ``{OA-WAM}: Object-addressable world action model for
  robust robot manipulation,'' \emph{arXiv preprint arXiv:2605.06481}, 2026.

\bibitem{qiu2026foresight}
L.~Qiu, Y.~Li, Y.~Chen, Y.~Ge, Y.~Ge, and X.~Liu, ``Making foresight
  actionable: Repurposing representation alignment in world action models,''
  \emph{arXiv preprint arXiv:2606.12217}, 2026.

\bibitem{pang2025reviwo}
J.-C. Pang, N.~Tang, K.~Li, Y.~Tang, X.-Q. Cai, Z.-Y. Zhang, G.~Niu,
  M.~Sugiyama, and Y.~Yu, ``Learning view-invariant world models for visual
  robotic manipulation,'' in \emph{International Conference on Learning
  Representations (ICLR)}, 2025.

\bibitem{heo2026anycamvla}
H.~Heo, S.~Woo, S.~M. Kim, J.~Kim, J.~Lee, Y.~Lee, and Y.~M. Kim,
  ``{AnyCamVLA}: Zero-shot camera adaptation for viewpoint robust
  vision-language-action models,'' \emph{arXiv preprint arXiv:2603.05868},
  2026.

\bibitem{gu2026vistabot}
S.~Gu, Y.~Zheng, W.~Li, Y.~Zheng, Y.~Feng, X.~Li, Y.~Chen, P.~Li, and W.~Ding,
  ``{VistaBot}: View-robust robot manipulation via spatiotemporal-aware view
  synthesis,'' \emph{arXiv preprint arXiv:2604.21914}, 2026.

\bibitem{jiang2025camera}
T.~Jiang, J.~Ji, X.~Tan, J.~Fang, A.~Bhattad, V.~Guizilini, and M.~R. Walter,
  ``Do you know where your camera is? {View}-invariant policy learning with
  camera conditioning,'' \emph{arXiv preprint arXiv:2510.02268}, 2025.

\bibitem{tian2024vista}
S.~Tian, B.~Wulfe, K.~Sargent, K.~Liu, S.~Zakharov, V.~Guizilini, and J.~Wu,
  ``View-invariant policy learning via zero-shot novel view synthesis,'' in
  \emph{Conference on Robot Learning (CoRL)}, 2024.

\bibitem{chen2024roviaug}
L.~Y. Chen, C.~Xu, K.~Dharmarajan, M.~Z. Irshad, R.~Cheng, K.~Keutzer,
  M.~Tomizuka, Q.~Vuong, and K.~Goldberg, ``{RoVi-Aug}: Robot and viewpoint
  augmentation for cross-embodiment robot learning,'' in \emph{Conference on
  Robot Learning (CoRL)}, 2024.

\bibitem{cai2026beyond}
B.~Cai, Q.~Liang, J.~Li, S.~Weng, Z.~Zhang, T.~Lin, X.~Chen, W.~Zhang, J.~Mao,
  W.~Xu, B.~Yang, J.~Liang, J.~Cai, and R.~Xu, ``Beyond viewpoint
  generalization: What multi-view demonstrations offer and how to synthesize
  them for robot manipulation?'' \emph{arXiv preprint arXiv:2603.26757}, 2026.

\bibitem{zhang2025robustvla}
H.~Zhang, S.~Zhang, J.~Jin, Q.~Zeng, R.~Li, and D.~Wang, ``{RobustVLA}:
  Robustness-aware reinforcement post-training for vision-language-action
  models,'' \emph{arXiv preprint arXiv:2511.01331}, 2025.

\bibitem{yang2024equibot}
J.~Yang, Z.-a. Cao, C.~Deng, R.~Antonova, S.~Song, and J.~Bohg, ``{EquiBot}:
  {SIM}(3)-equivariant diffusion policy for generalizable and data efficient
  learning,'' in \emph{Conference on Robot Learning (CoRL)}, 2024.

\bibitem{wang2024equivariant}
D.~Wang, S.~Hart, D.~Surovik, T.~Kelestemur, H.~Huang, H.~Zhao, M.~Yeatman,
  J.~Wang, R.~Walters, and R.~Platt, ``Equivariant diffusion policy,'' in
  \emph{Conference on Robot Learning (CoRL)}, 2024.

\bibitem{jeong2026vila}
Y.~Jeong, J.~Chun, and T.~Kim, ``Learning to act robustly with view-invariant
  latent actions,'' \emph{arXiv preprint arXiv:2601.02994}, 2026.

\bibitem{laine2017temporal}
S.~Laine and T.~Aila, ``Temporal ensembling for semi-supervised learning,'' in
  \emph{International Conference on Learning Representations (ICLR)}, 2017.

\bibitem{tarvainen2017mean}
A.~Tarvainen and H.~Valpola, ``Mean teachers are better role models:
  Weight-averaged consistency targets improve semi-supervised deep learning
  results,'' in \emph{Advances in Neural Information Processing Systems
  (NeurIPS)}, 2017.

\bibitem{sohn2020fixmatch}
K.~Sohn, D.~Berthelot, N.~Carlini, Z.~Zhang, H.~Zhang, C.~A. Raffel, E.~D.
  Cubuk, A.~Kurakin, and C.-L. Li, ``{FixMatch}: Simplifying semi-supervised
  learning with consistency and confidence,'' in \emph{Advances in Neural
  Information Processing Systems (NeurIPS)}, 2020.

\bibitem{chen2020group}
S.~Chen, E.~Dobriban, and J.~H. Lee, ``A group-theoretic framework for data
  augmentation,'' \emph{Journal of Machine Learning Research}, vol.~21, no.
  245, pp. 1--71, 2020.

\bibitem{yang2023dac}
S.~Yang, Y.~Dong, R.~Ward, I.~S. Dhillon, S.~Sanghavi, and Q.~Lei, ``Sample
  efficiency of data augmentation consistency regularization,'' in
  \emph{International Conference on Artificial Intelligence and Statistics
  (AISTATS)}, 2023.

\bibitem{karras2022edm}
T.~Karras, M.~Aittala, T.~Aila, and S.~Laine, ``Elucidating the design space of
  diffusion-based generative models,'' in \emph{Advances in Neural Information
  Processing Systems (NeurIPS)}, 2022.

\bibitem{liu2023libero}
B.~Liu, Y.~Zhu, C.~Gao, Y.~Feng, Q.~Liu, Y.~Zhu, and P.~Stone, ``{LIBERO}:
  Benchmarking knowledge transfer for lifelong robot learning,'' in
  \emph{Advances in Neural Information Processing Systems (NeurIPS), Datasets
  and Benchmarks Track}, 2023.

\end{thebibliography}

\end{document}